\documentclass[letterpaper, 10 pt]{IEEEtran}
\usepackage{graphicx}
\usepackage{mathtools}
\usepackage{blkarray}
\usepackage{multirow}
\usepackage{diagbox}
\usepackage{fullpage}
\usepackage{url}

\usepackage{amsmath,graphicx,color,amsfonts,subfigure}
\usepackage[noend]{algorithmic}
\usepackage{version,xspace}
\usepackage{tikz}
\usetikzlibrary{automata, positioning}
\usetikzlibrary{circuits.logic.US,circuits.logic.IEC,fit}

\usepackage{centernot}
\usepackage{dcolumn}

\usepackage{multirow}

\usepackage{hyperref}%
\hypersetup{colorlinks=true, linkcolor=blue, breaklinks=true, urlcolor=blue}

\usepackage{caption}
\usepackage[flushleft]{threeparttable}

\usetikzlibrary{calc}

\def\vspecdots{\vbox{\baselineskip=2pt \lineskiplimit=0pt
    \kern2pt \hbox{.}\hbox{.}\hbox{.}}}
\newcommand{\until}[1]{\{1,\dots, #1\}}

\newcommand{\setdef}[2]{\{#1 \; | \; #2\}}

\newcommand{\diag}[1]{\ensuremath{\operatorname{diag}(#1)}}
\newcommand{\Vect}[1]{\ensuremath{\operatorname{vec}(#1)}}

\newcommand{\kron}{\operatorname{\otimes}}

\newcommand{\fmt}{m}
\newcommand{\FMT}{M}
\newcommand{\MFMT}{\mathcal{M}}

\newcommand\oprocendsymbol{\hbox{$\square$}}
\newcommand\oprocend{\relax\ifmmode\else\unskip\hfill\fi\oprocendsymbol}

\renewcommand{\baselinestretch}{0.9815}

\def \bs {\boldsymbol}
\def \mc {\mathcal}

\newcommand{\G}{\mc G}
\newcommand{\V}{V}
\newcommand{\EE}{\mathcal{E}}
\newcolumntype{d}[1]{D{.}{.}{#1}}

\DeclareSymbolFont{bbold}{U}{bbold}{m}{n}
\DeclareSymbolFontAlphabet{\mathbbold}{bbold}
\newcommand{\vect}[1]{\mathbbold{#1}}
\newcommand{\vectorones}[1][]{\vect{1}_{#1}}

\newtheorem{theorem}{Theorem}
\newtheorem{corollary}{Corollary}
\newtheorem{lemma}{Lemma}
\newtheorem{remark}{Remark}

\newtheorem{problem}{Problem}
\newtheorem{definition}{Definition}
\newenvironment{proof}[1][Proof]{\begin{trivlist}
\item[\hskip \labelsep {\bfseries #1}]}{\end{trivlist}}
\renewcommand{\theenumi}{(\roman{enumi}}

\usepackage[margin=.75in]{geometry}

\title{
Robotic Surveillance Based on the Meeting Time of Random Walks
  \thanks{This work has been supported in part by Air Force Office of Scientific Research award FA9550-15-1-0138.}}

\author{Xiaoming Duan, Mishel George, Rushabh Patel, and Francesco Bullo
\thanks{Xiaoming Duan and Francesco Bullo are with the
    Mechanical Engineering Department and the Center of Control, Dynamical
    Systems and Computation, UC Santa Barbara, CA 93106-5070, USA. {\tt\small\{xiaomingduan,bullo\}@ucsb.edu}}
\thanks{Mishel George is with DoorDash, San Francisco, CA, USA. {\tt\small mishelgeorge@gmail.com}}
\thanks{Rushabh Patel is with Systems Technology, Inc., Hawthorne, CA, USA. {\tt\small patelrush@gmail.com}}}

\begin{document}
\graphicspath{{./figures/}}
\maketitle
\renewcommand{\theenumi}{(\roman{enumi})}%

\begin{abstract}
  This paper analyzes the \emph{meeting time} between a pair of pursuer and
  evader performing random walks on digraphs. The existing bounds on the
  meeting time usually work only for certain classes of walks and cannot be
  used to formulate optimization problems and design robotic
  strategies. First, by analyzing multiple random walks on a common graph as a
  single random walk on the Kronecker product graph, we provide the first
  closed-form expression for the expected meeting time in terms of the
  transition matrices of the moving agents.  This novel expression leads to
  necessary and sufficient conditions for the meeting time to be finite and
  to insightful graph-theoretic interpretations.  Second, based on the
  closed-form expression, we setup and study the minimization problem for
  the expected capture time for a pursuer/evader pair. We report
  theoretical and numerical results on basic case studies to show the
  effectiveness of the design.
\end{abstract}


\section{Introduction}\label{sec:Intro}
\subsection{Problem description and motivation}
In this paper, we examine the meeting time between two moving agents modeled by discrete-time Markov chains. This
problem is motivated by a pursuer trying to intercept a moving evader. The
meeting time, in the context of this paper, describes the average time till
a first encounter occurs between the pursuer and the evader given initial
positions of the pursuer and the evader. This notion of two adversarial
mobile agents wherein one of the agents is trying to intercept the other
appears under several names: pursuit-evasion
games~\cite{RV-OS-HJK-DHS-SS:02}, predator-prey
interactions~\cite{CC-AF-TR:09}, cops and robbers games~\cite{MA-MF:84,
  AB-PG-GH-JK:09} and princess-monster games~\cite{SA-RF-RL-GO:08}. Our
primary motivation is the design of stochastic surveillance strategies for
quickest detection of the mobile intruder. Single and multi-agent
surveillance strategies appear in environmental
monitoring~\cite{VS-FP-FB:11za}, minimizing emergency vehicle response
times~\cite{THB-JSK:02}, traffic routing and border
patrol~\cite{MP-CD-VD-AK-YW:03}. More broadly random walks on
networks appear in many areas of research: they are used to describe
effective resistance in electrical networks~\cite{PGD-JLS:84}, for
link-prediction and information propagation in social
networks~\cite{LB-JL:11}, and in designing search algorithms on
networks~\cite{MSS-DT-SB:14}. Aside from our proposed application to
stochastic surveillance, the meeting time has direct applications to
information flow in distributed networks~\cite{SRE-TB:16},
self-stabilization of tokens~\cite{AI-MJ:90} and measuring similarity of
objects~\cite{GJ-JW:02}.
\subsection{Literature review}
Early interest in meeting times was motivated by applications to self-stabilizing token management schemes~\cite{PT-PW:91}. In a token management scheme, only one of the many processors on a distributed network is enabled to change state or perform a particular task, and this processor is said to possess the token. If two tokens meet then they collapse into a single token. Israeli and Jalfon suggest a scheme in which the token is passed randomly to a neighbor~\cite{AI-MJ:90}. In a general connected undirected graph they were able to obtain an exponential bound for the meeting time of two tokens in terms of the maximum degree and the diameter of the graph. Coppersmith \emph{et al}.~\cite{DC-PR-PW:93} improved the bound to be polynomial in the number of nodes by bounding the meeting time in terms of the pairwise hitting time from the starting nodes of the tokens to hidden vertices.
Bshouty \emph{et al}.~\cite{NHB-LH-JWG:99} obtain a bound on the meeting time of several such tokens in terms of the meeting time of two tokens. Bounds for meeting times of two identical independent continuous-time reversible Markov chains in terms of the pairwise hitting times of the chain are mentioned in~\cite{DJA:91}. 

Several metrics have been used to describe single and multiple random walks of graphs. 
 One closely related metric is the hitting time which is the time taken by a single random walker to travel between nodes of a graph. The hitting time of a finite irreducible Markov chain first appeared in~\cite{JGK-JLS:76}. 
Several bounds have been obtained and many closed-form formulas exist to compute the hitting time for various graph topologies~\cite{JJH:13}.  The authors in~\cite{RP-AC-FB:14k} obtain a closed-form solution for the hitting time of multiple random walkers.
Another related notion is the coalescence time of multiple random walkers widely studied in the context of voter models~\cite{CC-RE-HO-TR:13}. Two random walks coalesce into one when they share the same node. Bounds for the coalescence time in terms of the worst case pairwise hitting times are discussed in~\cite{DA-JAF:02}. More recently, Cooper \emph{et al}. bounded the coalescence time using the second largest eigenvalue of the transition matrix~\cite{CC-RE-HO-TR:13}.

Stochastic vehicle routing strategies have the desirable property that an intruder cannot predictably plan a path to avoid surveillance agents. The authors in~\cite{RP-PA-FB:14b,MG-SJ-FB:17b,XD-MG-FB:17o} use Markov chains to design surveillance strategies. A novel convex optimization formulation is used to design strategies with minimum mean hitting time in~\cite{RP-PA-FB:14b}. In~\cite{PA-FB:15e} the mean hitting time in conjunction with multiple parallel CUSUM algorithms at various nodes of interest in the graph are used to describe a policy which ensures quickest average time to the detection of anomalies. In the strategies mentioned in these works the intruder/anomaly is assumed to be stationary. The policies for surveillance derived in this paper are for mobile intruders modeled by Markov chains.

\subsection{Contributions}
Given the above, there are several contributions in this paper. First, we
provide a set of necessary and sufficient conditions which characterize
when the meeting times between a single pursuer and a single evader is
finite for arbitrary Markov chains. To the best of our knowledge the bounds
in the literature were obtained for meeting times between ergodic Markov
chains where the meeting times are guaranteed to be finite. We extend the
notion to generic transition matrices as opposed to equal-neighbor models,
and we discuss when the meeting times are finite based on the existence of
walks of equal length to common nodes.  Second, we provide a closed-form
solution to the meeting time of two independent Markov chains by utilizing
the Kronecker product of the transition matrices. Third, we use this
closed-form expression to perform theoretical and simulation studies and
design fast Markov chain strategies for the pursuer to capture, in minimum
expected time, different moving evaders in different prototypical
graphs. In particular, in ring and complete graphs, we rigourously show a few qualitative features of the design. For example, being fast for the pursuer is not always necessary and the mean capture time may be indifferent to the pursuer's strategy for certain evaders.




To the best of our knowledge, this paper provides the first closed-form solutions for the computation of the meeting time between two agents moving on a graph according to discrete-time Markov chains. Two closely related references are as follows: first, a system of equations for computing meeting times for independent identical random walks on graphs with irreducible transition matrices, where the transition matrices are limited to equal-neighbor weights, were obtained using Laplace transform techniques in~\cite{TO:15}. Second, Kronecker products and vectorization techniques have been used to compute the Simrank of information networks which has interpretations in terms of meeting times~\cite{CL-JH-GH-XJ-YS-YY-TW:10}. In contrast, we consider absolutely generic transition matrices which need not be identical.

\subsection{Organization}
This paper is organized as follows. In Section~\ref{sec:Notation} we introduce notation that is used throughout the paper and review useful concepts. In Section~\ref{sec:Single_pursuer_evader} we introduce our formulation for the meeting times of pairs of Markov chains, and also define sets of pairs of matrices for which finite meeting times exist. In Section~\ref{sec:simulation} we present simulation results on fast Markov chain strategies for the mobile pursuer. Finally, we conclude the paper in Section~\ref{sec:conclusion}.

\section{Notation and Preliminaries} \label{sec:Notation}
In this section, we provide an
overview of Markov chains and introduce notation that will be used throughout the paper to deal with vectors, matrices, and the Kronecker product.
\subsection{Markov chains} \label{subsec:MarkChains}
A finite-state \emph{Markov chain} is a sequence of random variables taking values in the \emph{finite} set $\until{n}$ with the Markov property, i.e., the future state depends only on the current state.

Let $X_t \in \until{n}$ denote the location of a random walker at time $t \in \{0,1,$ $2,\dots\}$. A discrete-time Markov chain is \emph{time-homogeneous} if $\mathbb{P}[X_{t+1}=j\,|\,X_t =  i]=p_{i,j}$ for all $i,j\in\until{n}$ and $t \geq0$, where $P =[p_{i,j}] \in \mathbb{R}^{n  \times n}$ is the \emph{transition matrix} of the Markov chain. By definition, each transition matrix $P$ is row-stochastic,  i.e., $P \vectorones[n] = \vectorones[n]$, where $\vectorones[n]$ is a vector of $1$'s in dimension $n$. The period of a state $i$ is defined as the greatest common divisor of all $t$ in $\{t\geq 1\,|\,\mathbb{P}[X_t=i\,|\,X_0=i]\neq 0\}$. A state whose period is one is referred to as an \emph{aperiodic} state, and a Markov chain is aperiodic if all of the states are aperiodic. All states in a \emph{communicating class} (defined below) share the same period. For more details on Markov chains refer \cite{JGK-JLS:76}.

For two states $i$ and $j$ of a Markov chain, state $i$ \emph{communicates} with $j$ if  $\mathbb{P}[X_t=j\,|\,X_0=i] \neq 0$ for some $t>0$. A subset of states $X \subset \until{n}$ forms a \emph{communicating class} if for every state $i,j \in X$ the states communicate with each other, i.e., $\mathbb{P}[X_{t}=j\,|\,X_0=i]\neq 0$ and $\mathbb{P}[X_{t^{'}}=i\,|\,X_0=j]\neq 0$ for some $t,t'\geq0$. An \emph{absorbing class} $A$ of a Markov chain is a communicating class such that the probability of escaping the set is zero, i.e., $\mathbb{P}[X_{t}=j\,|\,X_0=i]= 0$ for all $t>0$ for all $i \in A, j\notin A$. A communicating class that is not absorbing is a \emph{transient class}. In general, a Markov chain will have multiple absorbing and transient classes. If a Markov chain has only a single absorbing class then it is referred to as a \emph{single absorbing Markov chain}.

If a discrete-time Markov chain with transition matrix $P$ is single absorbing, then a unique stationary distribution $\bs \pi\in\mathbb{R}^n_{\geq0}$ exists, which satisfies $\sum_{i=1}^n \bs \pi_i = 1$ and $\bs \pi^\top P = \bs \pi^\top$. A Markov chain is \emph{irreducible} if the absorbing class is the entire set of states $\until{n}$. A discrete-time Markov chain is said to be \emph{ergodic} if it is irreducible and aperiodic.

\subsection{Kronecher product} \label{subsec:KronProps}
For two matrices $A \in \mathbb{R}^{n \times m}$ and $B \in \mathbb{R}^{q \times r}$, the Kronecker product $A\kron B$ is an $nq \times mr$ matrix given by
\begin{align*}
  A\kron B = \left[
  \begin{array}{ccc}
    a_{1,1} B & \dots & a_{1,m} B \\
    \vdots& \ddots  &\vdots \\
    a_{n,1} B &  \ddots& a_{n,m}B
  \end{array} \right].
\end{align*}
%
A few
properties of the Kronecker product and vectorization of matrices are summarized in the following lemma.
\begin{lemma}[Kronecher product and vectorization identities] \label{lem:product-identities}
  Given matrices $A,B,C$ and $D$ of appropriate dimensions, the following identities hold:
  \begin{enumerate}
  \item $(A \kron B) (C \kron D) = (A  C) \kron (B  D)$,
  \item $(B^\top \kron A) \Vect{C}= \Vect{ACB}$.
  \end{enumerate}
\end{lemma}

\subsection{Markov chains on graphs}
In this paper, we consider a strongly connected digraph $\G = (\V,\EE)$ with the node set $V =\until{n}$ and the edge set $\EE \subset \V \times \V$. The transition matrix $P=[p_{i,j}]$ of a Markov chain on $\G$ satisfies that $p_{i,j}\geq0$ if $(i,j) \in \EE$ and $p_{i,j}=0$ otherwise. 
There exists a \emph{walk} of \emph{length} $\ell$ from node $i_1$ to node $i_{\ell+1}$ for $P$ if there exists a sequence of nodes $i_1,i_2,\dots,i_{\ell+1}$ such that $p_{i_k,i_{k+1}}>0$ for $1\leq k \leq \ell$.



The following lemma shall be used later.
\begin{lemma} [Convergence of substochastic matrices {\cite[Lemma~2.2]{RP-AC-FB:14k}}]  \label{lem:substochastic}
  Let $P \in \mathbb{R}^{n \times n}$ be a row-substochastic matrix, then $P$ has spectral radius less than $1$ if and only if for every node with row-sum $1$ there exists a walk to a node with row-sum less than $1$.
\end{lemma}

\section{Meeting time of two randomly moving agents} \label{sec:Single_pursuer_evader}
In this section, we formulate the meeting time between a pursuer and an evader moving according to discrete-time Markov chains. We provide necessary and sufficient conditions for the finiteness of the meeting times given initial starting positions of the agents on the graph. 
\subsection{The meeting time of two Markov chains}
Consider the pursuer and evader performing random walks on a strongly connected graph $\G=(V,\EE)$ with the node set $\V =\until{n}$ and $\EE \subset \V \times \V$.  The transition matrices $P_\textnormal{p}$ of the pursuer and $P_\textnormal{e}$ of the evader satisfy $p_{i,j}^{(\textnormal{p})},p_{i,j}^{(\textnormal{e})} \geq 0$ if $(i,j) \in \EE$ and $p_{i,j}^{(\textnormal{p})},p_{i,j}^{(\textnormal{e})} = 0$ if $(i,j) \notin \EE$.

Let $X_t^{(\textnormal{p})},X_t^{(\textnormal{e})} \in\until{n}$  be the locations of the two agents at time $t\in\{0,1,2,\dots\}$, respectively. For any two starting nodes $i$ and $j$, the \emph{first meeting time from $i$ and $j$}, denoted by  $T_{i,j}$, is the first time that two random walkers meet at a common node when starting from nodes $i$ and $j$. Formally,
\begin{equation*}
  T_{i,j} =\min \setdef{t\geq 1}{X_t^{(\textnormal{p})} = X_t^{(\textnormal{e})}, X_0^{(\textnormal{p})} = i \text{ and } X_0^{(\textnormal{e})}=j}.
\end{equation*}
Note that the first meeting time can be infinite and it is easy to
construct examples in which the two agents never meet. Moreover, by definition, if the two agents are at the same location initially, i.e., $i=j$, then $T_{i,j}$ is the first time they meet again. Let
$\fmt_{i,j} = \mathbb{E}[T_{i,j}]$ be the expected first
meeting time starting from nodes $i$ and $j$. For the sake of brevity, we shall refer to the expected first meeting time as just the meeting time.

\begin{theorem}[The meeting time of two Markov chains]\label{thm:necessary_sufficient_finite}
  Consider two Markov chains with transition matrices $P_\textnormal{p}$ and
  $P_\textnormal{e}$ defined on a digraph $\G=(V,\EE)$ with the node set $V = \until{n}$.  The
  following statements are equivalent:
  \begin{enumerate}
  \item\label{thm11} for each pair of nodes $i,j$, the meeting time
    $\fmt_{i,j}$ from nodes $i$ and $j$ is finite;
  \item\label{thm12} for each pair of nodes $i,j$, there exists a node $k$ and
    a length $\ell$ such that a walk of length $\ell$ exists
    from $i$ to $k$ for $P_\textnormal{p}$ and a walk of length $\ell$ exists from
    $j$ to $k$ for $P_\textnormal{e}$;
  \item\label{thm13} for each pair of nodes $i,j$, there exists a
    walk for the stochastic matrix \textnormal{$P_\textnormal{e} \kron
    P_\textnormal{p}$} from node $(i,j)$ to a node $(k,k)$ in the Kronecker graph, for some
    $k\in\until{n}$;

  \item\label{thm14} the sub-stochastic matrix \textnormal{$(P_\textnormal{e} \kron P_\textnormal{p})E$}
    has spectral radius less than $1$ and the vector of meeting times is given by
    \textnormal{
    \begin{align} \label{eq:Vector2agent}
      \Vect{\FMT} = (I_{n^2}-(P_\textnormal{e}\kron P_\textnormal{p})E)^{-1} \vectorones[n^2],
    \end{align}}
    where $\FMT=[\fmt_{i,j}]$ and $E =I_{n^2}-\diag{\Vect{I_n}}$.
  \end{enumerate}
\end{theorem}

\begin{proof}
  For the nodes $i$ and $j$, the first meeting time satisfies the recursive formula
  \begin{equation*}
    T_{i,j} =
    \begin{cases}
      1, & \text{w.p.}~\sum_k p_{i,k}^{(\textnormal{p})}p_{j,k}^{(\textnormal{e})},\\
      T_{k_1,h_1}+1, & \text{w.p.}~ p_{i,k_1}^{(\textnormal{p})}p_{j,h_1}^{(\textnormal{e})}, k_1\neq h_1.
    \end{cases}
  \end{equation*}
  Taking the expectation we have
  \begin{equation}\label{eq:meetingtime}
    \begin{split}
      \fmt_{i,j}  &=\sum_{k} p_{i,k}^{(\textnormal{p})}p_{j,k}^{(\textnormal{e})} + \sum_{k_1 \neq h_1 } p_{i,k_1}^{(\textnormal{p})}p_{j,h_1}^{(\textnormal{e})} (\fmt_{k_1,h_1} +1),\\
      &= \sum_{k_1}\sum_{h_1}p_{i,k_1}^{(\textnormal{p})} p_{j,h_1}^{(\textnormal{e})} + \sum_{k_1 \neq h_1 } p_{i,k_1}^{(\textnormal{p})}p_{j,h_1}^{(\textnormal{e})} \fmt_{k_1,h_1} , \\
      &=1 + \sum_{k_1 \neq h_1 } p_{i,k_1}^{(\textnormal{p})}p_{j,h_1}^{(\textnormal{e})} \fmt_{k_1,h_1}\\
      &=1 + \sum_{k_1,h_1} p_{i,k_1}^{(\textnormal{p})} \fmt_{k_1,h_1} p_{j,h_1}^{(\textnormal{e})}  - \sum_{k=1}^n p^{(\textnormal{p})}_{i,k}\fmt_{k,k}p^{(\textnormal{e})}_{j,k}.
    \end{split}
  \end{equation}
  We write \eqref{eq:meetingtime} in matrix form as
    \begin{equation}\label{eq:meetingtimematrix}
      \FMT = \vectorones[n] \vectorones[n]^\top + P_\textnormal{p}(\FMT-\diag{\FMT})P_\textnormal{e}^\top,
  \end{equation}
  where $\diag{\FMT} \in \mathbb{R}^{n\times n}$ is a diagonal matrix with only the diagonal elements of $\FMT$ . 
  Rewriting \eqref{eq:meetingtimematrix} in vector form and using properties in Lemma~\ref{lem:product-identities} gives
  \begin{equation*}
    \begin{split}
      \Vect{\FMT} &= \vectorones[n^2] + (P_\textnormal{e}\kron P_\textnormal{p})(\Vect{\FMT}  -\Vect{\diag{\FMT}}),\\
     &= \vectorones[n^2] + (P_\textnormal{e}\kron P_\textnormal{p})(I_{n^2}-\diag{\Vect{I_n}}) \Vect{M},\\
          &= \vectorones[n^2] + (P_\textnormal{e}\kron P_\textnormal{p})E \Vect{M}.
    \end{split}
  \end{equation*}
   If the matrix $I_{n^2}-(P_\textnormal{e}\kron P_\textnormal{p})E$ is invertible, then we have a unique solution to the meeting times.

  We shall now show that the finiteness of meeting times as in \ref{thm11} is equivalent to the existence of walks of equal length to common nodes as mentioned in \ref{thm12} and in \ref{thm13}, which guarantees invertibility of $I_{n^2}-(P_\textnormal{e}\kron P_\textnormal{p})E$ in $(iv)$.


  We first prove that \ref{thm11} $\implies$ \ref{thm12} by contrapositive. Suppose there exists a pair of nodes $i$ and $j$ such that there exists no walk of equal length to any node in $V$, then the agents never meet and thus the meeting time cannot be finite. Therefore, we have \ref{thm11} $\implies$ \ref{thm12}.

  Next, we show that \ref{thm12} $\iff$ \ref{thm13}. The Kronecker product of the transition matrices gives a joint transition matrix for the agents over the set of nodes $V \times V$. The entry in the matrix $P_\textnormal{e} \kron P_\textnormal{p}$ corresponding to the node $(i,j)$ represents the states $X^{(\textnormal{p})}=i$ and $X^{(\textnormal{e})}=j$ \cite{PMW:62}. The statement $(ii)$ ensures the existence of a node $k$ for every pair $(i,j)$ which is reachable by a walk of equal length from $i$ in $P_\textnormal{p}$ and $j$ in $P_\textnormal{e}$. This condition is equivalent to the node $(k,k)$ being reachable from the pair $(i,j)$ on the Kronecker product of the two Markov chains \cite[Proposition 1]{FH-CATJ:66}.

  Next, we show that \ref{thm13} $\implies$ \ref{thm14}. The stochastic matrix $P_\textnormal{e} \kron P_\textnormal{p}$ has a walk from any node $(i,j)$ to some node $(h_1,k_1)$ where $\mathbb{P}[X_1^{(\textnormal{e})}=k,X_1^{(\textnormal{p})}=k\,|\, X_0^{(\textnormal{e})}=h_1,X_0^{(\textnormal{p})}=k_1]\neq 0$ as there exists a walk from $(i,j)$ to $(k,k)$ for some $k$. Note that $(P_\textnormal{e} \kron P_\textnormal{p})E$ is obtained by setting the columns of $P_\textnormal{e} \kron P_\textnormal{p}$ corresponding to nodes of the form $(k,k)$ to $0$. Therefore, the row corresponding to $(h_1,k_1)$ has row-sum strictly less than 1. Therefore every node $(i,j)$ has a walk to a node whose corresponding row-sum of the transition matrix is less than 1, which implies that the matrix $(P_\textnormal{e} \kron P_\textnormal{p}) E$ has spectral radius less than $1$ by virtue of Lemma~\ref{lem:substochastic}. From this we obtain equation \eqref{eq:Vector2agent} since $(iii)$ guarantees the existence of $(I_{n^2}-(P_\textnormal{e}\kron P_\textnormal{p})E)^{-1}$.

  Note that the existence of $\Vect{\FMT}$ in \ref{thm14} gives \ref{thm14}$\implies$ \ref{thm11}. Thus we have shown that \ref{thm11}~$ \implies$~\ref{thm12}~$\iff$~\ref{thm13}~$\implies$~\ref{thm14}~$\implies $~\ref{thm11}. Hence the four statements are equivalent.
\end{proof}

\begin{remark}
The finiteness of meeting times is not guaranteed even if both $P_p$ and $P_e$ are irreducible, and a simple example is given in Fig.~\ref{fig:twonodeexample}.
\end{remark}

\begin{figure}[http]
  \includegraphics[scale = 0.32]{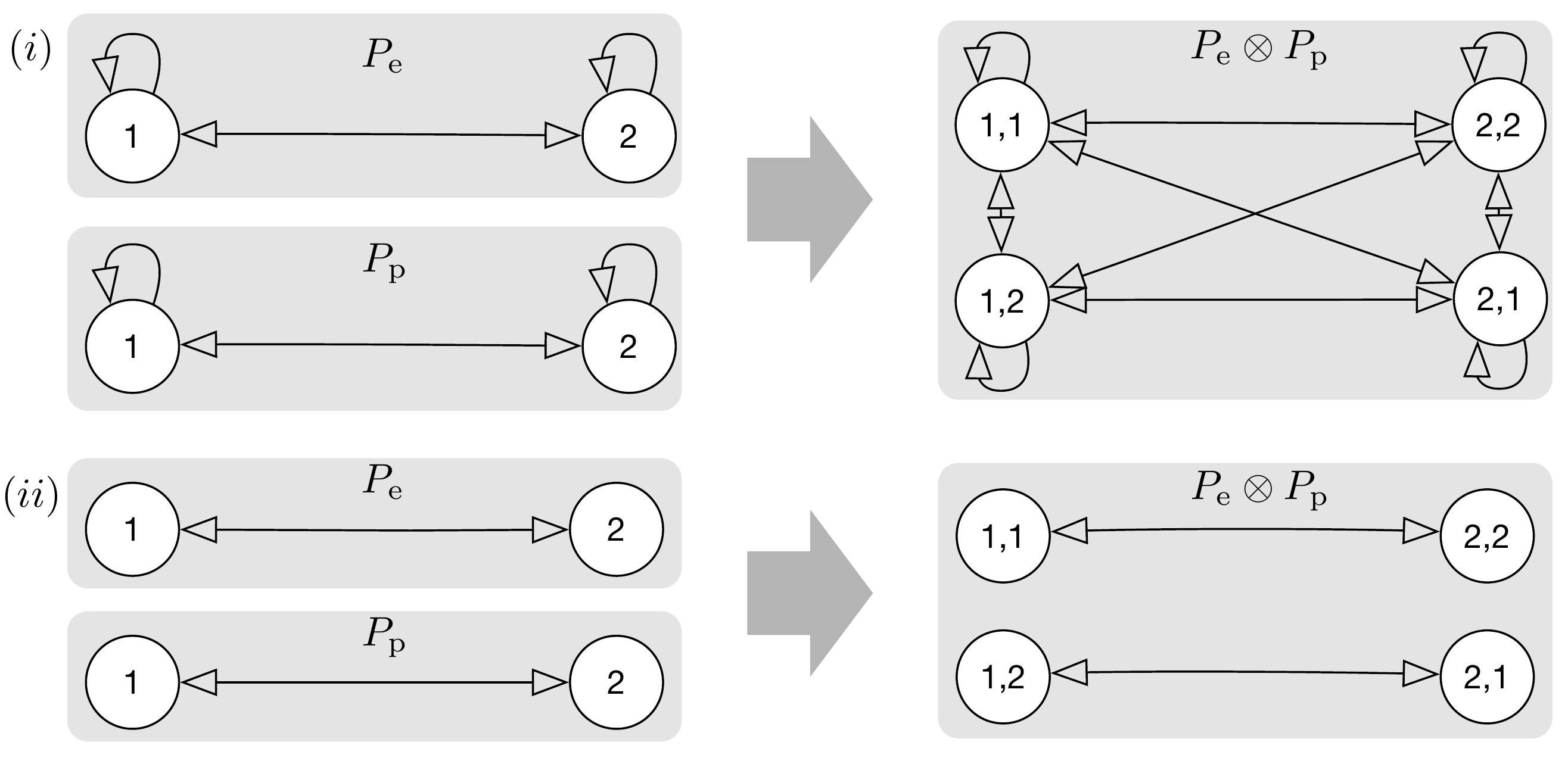}
  \caption{The pursuer-evader pair in $(i)$ has finite meeting times as every node has a walk to the common nodes $(1,1)$ and $(2,2)$ in the Kronecker graph. However, in $(ii)$ there exists no walks to common nodes from $(1,2)$ and $(2,1)$.}\label{fig:twonodeexample}
\end{figure}

The necessary and sufficient conditions in Theorem~\ref{thm:necessary_sufficient_finite} give the most general
set of pairs of matrices for which finite meeting times
exist. Moreover, the closed-form expression \eqref{eq:Vector2agent} for pairwise meeting times enables one to design the optimal strategy for the surveillance agent to minimize the mean meeting time given the strategy of a moving intruder.

\subsection{Mean meeting time and relation to hitting times}

In this subsection, we introduce the \emph{mean meeting time} of two random walkers. We then show that the mean hitting time can be treated as a special case of mean meeting time where a mobile pursuer is faced with a stationary intruder.

\begin{definition}[Mean meeting time]\label{def:meetingtime}
  Consider two transition matrices $P_\textnormal{p}$ and $P_\textnormal{e}$ with the stationary distributions $\pi_\textnormal{p}$ and $\pi_\textnormal{e}$, the mean meeting time $\MFMT(P_\textnormal{p},P_\textnormal{e})$ is defined by
  \begin{equation}\label{eq:MFMT}
    \MFMT(P_\textnormal{p},P_\textnormal{e}) = \pi_\textnormal{p}^\top \FMT \pi_\textnormal{e} =(\pi_\textnormal{e} \kron \pi_\textnormal{p})^\top \Vect{\FMT},
  \end{equation}
  where $M$ is meeting time matrix of $P_\textnormal{p}$ and $P_\textnormal{e}$ .
\end{definition}

The mean meeting time \eqref{eq:MFMT} can also be written in element-wise form as follows,
  \begin{equation*}
      \MFMT(P_\textnormal{p},P_\textnormal{e}) = \sum_{i}\sum_{j}\pi_\textnormal{p}^{(i)}\pi_\textnormal{e}^{(j)} m_{i,j},
  \end{equation*}
where it is clear that the mean meeting time is the weighted sum of the pairwise meeting times with weights being the stationary distributions.

\begin{remark}
In Definition~\ref{def:meetingtime}, the uniqueness of the stationary distributions for $P_\textnormal{p}$ and $P_\textnormal{e}$ is not required. However, in order to compute the mean meeting time, one has to specify \emph{a} stationary distribution consistent with the Markov chain for $P_e$ and $P_p$, respectively.
\end{remark}

Our next result shows that the hitting times of a Markov chain are equal to the meeting times of the Markov chain and a stationary evader.

\begin{corollary}[Connections with hitting times and meeting times with stationary evader]
  Consider a stationary evader with distribution $\pi_e$ and a pursuer with an irreducible transition matrix $P_p$ and stationary distribution $\pi_p$, then the following properties hold:
  \begin{enumerate}
  \item  the meeting times between the stationary evader and the mobile pursuer are equal to the pairwise hitting times of $P_p$ and are given by
    \begin{equation}
      h_{i,j} = m_{i,j} = (\vect{e}_j \kron \vect{e}_i)^\top(I_{n^2} - (I_n \kron P_p)E)^{-1}\vectorones[n^2],
    \end{equation}
    where $h_{i,j}$ is the expected time to travel from node $i$ to node $j$ for $P_p$ and
  \item  the mean meeting time between the stationary evader and the pursuer is given by
    \begin{multline}\label{eq:meanmeetingst}
      \MFMT_\textnormal{stationary}(\pi_e,P_p)\\ = (\pi_e \kron \pi_p)^\top(I_{n^2} - (I_n \kron P_p)E)^{-1}\vectorones[n^2].
    \end{multline}
  \end{enumerate}
\end{corollary}
\begin{proof}
   The conclusion follows by observing that a stationary evader can be described by the identity transition matrix $I_n$. 
\end{proof}
\begin{remark}
When the stationary distribution of the evader $\pi_e$ is equal to the
stationary distribution of the pursuer $\pi_p$, the expression~\eqref{eq:meanmeetingst} for the
meeting time is also identical to the mean first passage time, also called
\emph{Kemeny constant}, of the Markov chain $P_p$~\cite[Theorem
  2.3(i)]{RP-AC-FB:14k}.
\end{remark}


\section{Applications to Robotic Surveillance}\label{sec:simulation}
In this section, we numerically minimize the mean
meeting time for the mobile pursuer given various strategies of the
intruder in various prototypical graphs. The optimization problem we are
interested in is as follows.

\begin{problem}\label{prob:meanmeetingtime}
(Minimization of the mean meeting time) Given a strongly connected directed graph $\mathcal{G}=(V,\mathcal{E})$, an irreducible Markov chain $P_e$ and the stationary distributions $\pi_e$ and $\pi_p$. Find $P_p$ which minimizes the mean meeting time $\MFMT(P_\textnormal{p},P_\textnormal{e})$, i.e., solve the following optimization problem:
\begin{align*}
& \underset{{P_p}}{\textup{minimize}}
& &(\pi_\textnormal{e} \kron \pi_\textnormal{p})^\top \Vect{\FMT}\\
& \textup{subject to}
&& \pi_p^\top P_p=\pi_p^\top,\\
&&&P_p\vectorones[n]=\vectorones[n],\\
&&&p_{i,j}^{(p)}\geq0,\quad \forall(i,j)\in E,\\
&&& p_{i,j}^{(p)}=0,\quad \forall(i,j)\notin E.
\end{align*}
\end{problem}

The mean meeting time measures in expectation how fast the pursuer is able to capture the evader when they start from different initial positions. By minimizing the mean meeting time, we obtain a fast pursuer given the strategy of the evader. Problem~\ref{prob:meanmeetingtime} is a nonconvex optimization problem with
the Kemeny constant minimization problem as a special case. We conduct the
numerical optimization using the KNITRO/TOMLAB package (with an
implementation of the sequential quadratic programming algorithm), where
the stationary distribution of $P_p$ is set to be the same as that of
$P_e$, i.e., $\pi_p=\pi_e$.


\subsection{Evader models}
We consider three different strategies for the evader, i.e., the random
walk (RW), the unpredictable evader modeled by the Markov chain with maximum entropy rate, and the fast evader modeled by the
Markov chain with minimum Kemeny constant.

In the random walk model, the evader transitions from her current location
to the neighbors (including the current location) with the same probability
that is equal to the reciprocal of the out-degrees. The random walk
maximizes the local uncertainty of the movement of the evader.

For the unpredictable evader, given the stationary distribution $\pi_e$, the evader solves the following convex optimization problem.
\begin{align*}
& \underset{{P_e}}{\textup{maximize}}
& &-\sum_{i=1}^n\sum_{j=1}^n\pi_e(i)p_{i,j}^{(e)}\log p_{i,j}^{(e)}\\
& \textup{subject to}
&& \pi_e^\top P_e=\pi_e^\top,\\
&&&P_e\vectorones[n]=\vectorones[n],\\
&&&p_{i,j}^{(e)}\geq0,\quad \forall(i,j)\in E,\\
&&& p_{i,j}^{(e)}=0,\quad \forall(i,j)\notin E.
\end{align*}
The unpredictable evader uses a Markov chain that has the maximum entropy rate with a given stationary distribution. The evader is unpredictable in terms of the sequence of locations that she visits~\cite{MG-SJ-FB:17b}.

For the fast evader, given the stationary distribution $\pi_e$, the evader solves the following nonconvex optimization problem.
\begin{align*}
& \underset{{P_e}}{\textup{minimize}}
& &\MFMT_\textnormal{stationary}(\pi_e,P_e)\\
& \textup{subject to}
&& \pi_e^\top P_e=\pi_e^\top,\\
&&&P_e\vectorones[n]=\vectorones[n],\\
&&&p_{i,j}^{(e)}\geq0,\quad \forall(i,j)\in E,\\
&&& P_e(i,j)=0,\quad \forall(i,j)\notin E,
\end{align*}
where $\MFMT_\textnormal{stationary}(\pi_e,P_e)$ is given in~\eqref{eq:meanmeetingst} with $\pi_p=\pi_e$. The fast evader uses a Markov chain that has minimum Kemeny constant. The evader is fast because the expected hitting time between pairs of locations on the graph is minimized~\cite{RP-PA-FB:14b}.

\subsection{Analysis and results for different graphs}
In this subsection, we consider different graph topology, i.e., ring, complete and grid, and solve for the best pursuer strategy $P_p$ numerically. Since Problem~\ref{prob:meanmeetingtime} is in general a nonconvex optimization problem, we consider relatively small graph sizes $n=5$ and $n=6$ for the ring and complete, and $n=9$ for the grid. In all the computations where a stationary distribution needs to be specified, we set the stationary distribution of the agents to be uniform, i.e., $\pi_\textnormal{p}=\pi_\textnormal{e}=\frac{1}{n}\vectorones[n]$.

\emph{Results for ring graphs}: Note that ring graphs possess Hamiltonian tours, which can be parameterized by Markov chains as permutation matrices with a uniform stationary distribution. Therefore, the fast evader on ring graphs follow Hamiltonian tours. However, depending on the number of nodes in the graph, the optimal strategies for the pursuer against the fast evader are different. When $n=5$, the optimal strategy for the pursuer given by the solver is a Hamiltonian tour in the opposite direction from that of the evader. This coincides with our intuition because walking in a different direction for the pursuer should make it faster to catch the evader. However, it turns out that staying stationary for the pursuer is equally good as walking in the opposite direction. This happens because the pursuer may miss the evader when walking in an opposite direction.  We formalize this observation as follows.

\begin{lemma}[Equally good strategies in a ring graph]
In a ring graph with an odd number of nodes, if the evader adopts a Hamiltonian tour, then staying stationary and the Hamiltonian tour in the opposite direction are equally good for a pursuer with uniform stationary distribution.
\end{lemma}
\begin{proof}
If the pursuer stays stationary with the distribution $\frac{1}{n}\vectorones[n]$, then the mean meeting time is the same as the Kemeny constant of the chain used by the fast evader, which is equal to $\frac{n+1}{2}$.

On the other hand, suppose the pursuer walks in the oppositive direction from the evader. By symmetry, we can fix the initial condition of the evader to be $X_0^{(e)}=1$ and vary the initial condition of the pursuer $X_0^{(p)}$. If $X_0^{(p)}=1$, then the mean meeting time is $n$; If $X_0^{(p)}>1$ is odd, then the mean meeting time is $\frac{X_0^{(p)}-1}{2}$; If $X_0^{(p)}$ is even, then the mean meeting time is $\frac{n+X_0^{(p)}-1}{2}$. Therefore, the mean meeting time can be calculated as
\begin{equation*}
\MFMT(P_\textnormal{p},P_\textnormal{e})=\frac{1}{n}(\sum_{i=1}^{\frac{n-1}{2}}\frac{n+2i-1}{2}+\sum_{i=1}^{\frac{n-1}{2}}i+n)=\frac{n+1}{2},
\end{equation*}
which is the same as in the case of staying stationary.
\end{proof}
When $n=6$, the optimal strategy given by the solver is to stay stationary. Different from the case when the number of nodes is odd, walking in a different direction from the evader is bad because there are certain pairs of initial positions starting from which the pursuer and the evader never meet, i.e., the mean meeting time is infinite.

In the ring graph, the RW and unpredictable evader uses the same chain where the evader moves to the neighbor nodes of her current position with equal probabilities. The optimal strategy for the pursuer given by the solver in these cases is a Hamiltonian tour regardless of the number of nodes in the graph. We summarize the results for the ring graph in Table~\ref{tb:ring}.

\begin{table}
\centering
  \begin{threeparttable}
\caption{Best Response for Pursuer in Ring Graphs}\label{tb:ring}
\begin{tabular}{|c|c|c|}
\hline
\diagbox[width=10em]{Number\\ of nodes}{Evader strategy}&Fast&RW/ Unpredictable \\ \hline
n=5   &stationary or $P_e^\top$ &\multirow{2}{*}{Hamiltonian tour} \\ \cline{1-2}
n=6&stationary &\\ \hline
\end{tabular}
    \begin{tablenotes}
      \small
       \item When the number of nodes is odd, the best strategy for the pursuer against a fast evader is to either stay stationary or being fast in the opposite direction; when the number of nodes is even, staying stationary is the best. When the evader is unpreditcable/slow, being fast is always good.
    \end{tablenotes}
  \end{threeparttable}
\end{table}

\emph{Results for the complete graphs}: Same as the ring graphs, complete graphs also possess Hamiltonian tours, and the fast evader on complete graphs follow a Hamiltonian tour on the graph. Therefore, the results for the fast evader in ring graphs carry over.

On the other hand, in complete graphs, the RW and unpredictable strategies are the same and equal to $\frac{1}{n}\vectorones[n]\vectorones[n]^\top$. The following result shows that if the evader adopts RW or unpredictable strategy in the complete graph, then the mean meeting time is $n$ regardless of pursuer's strategy.

\begin{lemma}[Strategy insensitivity in a complete graph]\label{lemma:complete}
If the evader's strategy on a complete graph is $\frac{1}{n}\vectorones[n]\vectorones[n]^\top$, then the mean meeting time between the evader and the pursuer is always $n$ regardless of pursuer's strategy.
\end{lemma}
\begin{proof}
If $P_e=\frac{1}{n}\vectorones[n]\vectorones[n]^\top$, then by \eqref{eq:meetingtime} we have
  \begin{equation*}
      \fmt_{i,j}  =1 + \frac{1}{n}\sum_{k_1=1}^np_{i,k_1}^{(\textnormal{p})}\sum_{h_1 \neq k_1 } \fmt_{k_1,h_1}.
  \end{equation*}
Therefore, the meeting time $m_{i,j}$ does not depend on $j$. Let $\tilde{m}_{i}=m_{i,j}$, then we further have
  \begin{equation*}
      \tilde{\fmt}_{i}  =1 + \frac{n-1}{n}\sum_{k_1=1}^np_{i,k_1}^{(\textnormal{p})} \tilde{\fmt}_{k_1}.
  \end{equation*}
Since the mean meeting time satisfies $\MFMT(P_\textnormal{p},P_\textnormal{e})=\pi_p^\top M \pi_e=\pi_p^\top\tilde{m}$ in this case, we have
  \begin{equation*}
      \MFMT =1 + \frac{n-1}{n}\MFMT.
  \end{equation*}
Therefore we have $\MFMT=n$.
\end{proof}
We summarize the results for the complete graph in Table~\ref{tb:complete}.

\begin{table}
\centering
\begin{threeparttable}
\caption{Best Response for Pursuer in Complete Graph}\label{tb:complete}
\begin{tabular}{|c|c|c|}
\hline
\diagbox[width=10em]{Number\\ of nodes}{Evader strategy}&Being fast&RW/Being unpredictable \\ \hline
n=5   &stationary or $P_e^\top$ &\multirow{2}{*}{Arbitrary} \\ \cline{1-2}
n=6&stationary &\\ \hline
\end{tabular}
    \begin{tablenotes}
       \item When the evader is fast, similar results as in ring graphs carry over. When the evader is unpreditcable/slow, any strategy for the purser is optimal.
    \end{tablenotes}
  \end{threeparttable}
\end{table}


\emph{Results for the grid}: We plot the optimal strategies for the pursuer against an RW evader, an unpredictable evader, and a fast evader in Fig.~\ref{fig:grid_RW}, Fig.~\ref{fig:grid_entropy}, and Fig.~\ref{fig:grid_kemeny}, respectively. In all these figures, the size of the nodes indicates the magnitude of the stationary distribution, and the transparency of the edges indicates the magnitude of the transition probability.

From Fig.~\ref{fig:grid_RW}, Fig.~\ref{fig:grid_entropy}, we observe that when faced with an unpredictable and slow evader, the pursuer tends to travel fast in the graph. Note that the stationary distributions of the evaders and thus of the pursuers in these two cases are different (determined by the equal-neighbor model for RW and $\frac{1}{n}\vectorones[n]$ for the unpredictable chain).  Specifically, the center node in Fig.~\ref{fig:grid_RW} has a higher value in the stationary distribution than that in Fig.~\ref{fig:grid_entropy}. Qualitatively, this difference forces the neighbor nodes of the center node to have positive transition probabilities to the center in the solution of the optimal pursuer in Fig.~\ref{fig:grid_RW}, whereas it is not the case in Fig.~\ref{fig:grid_entropy}.

\begin{figure}[htbp]
\begin{center}
\subfigure[ RW evader chain]{
\includegraphics[scale=0.38]{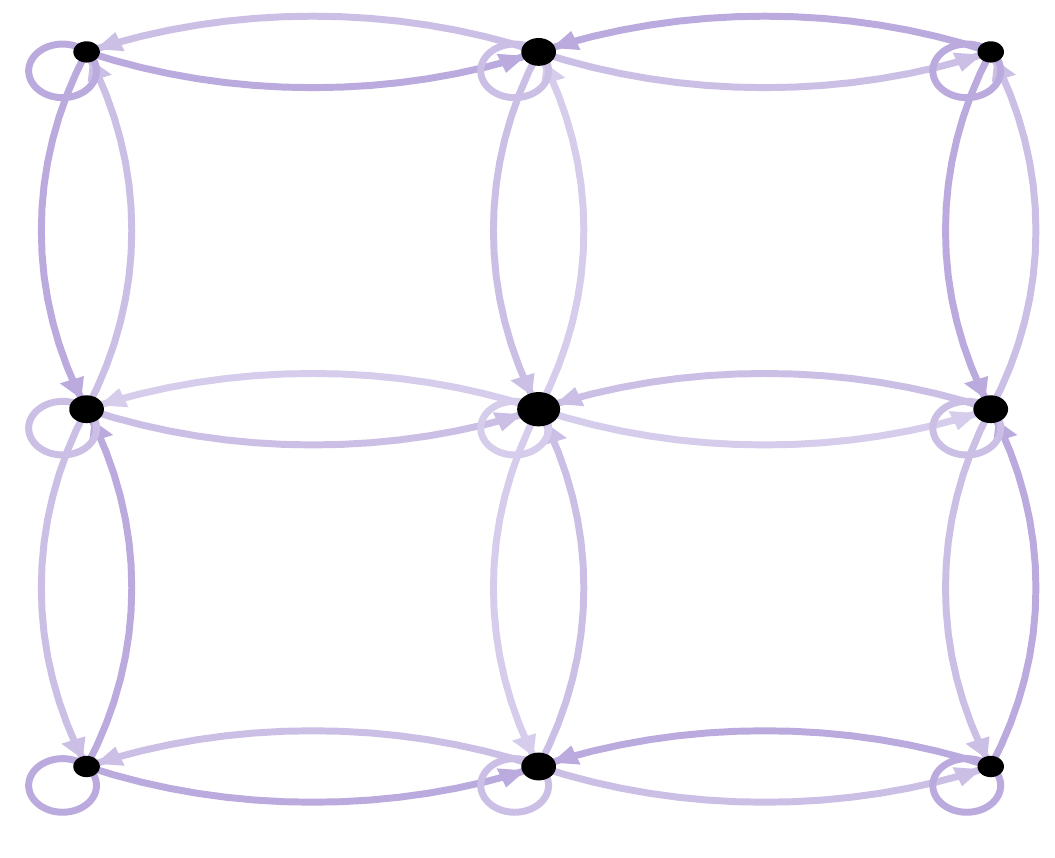}}
\subfigure[ Optimal pursuer chain]{
\includegraphics[scale=0.38]{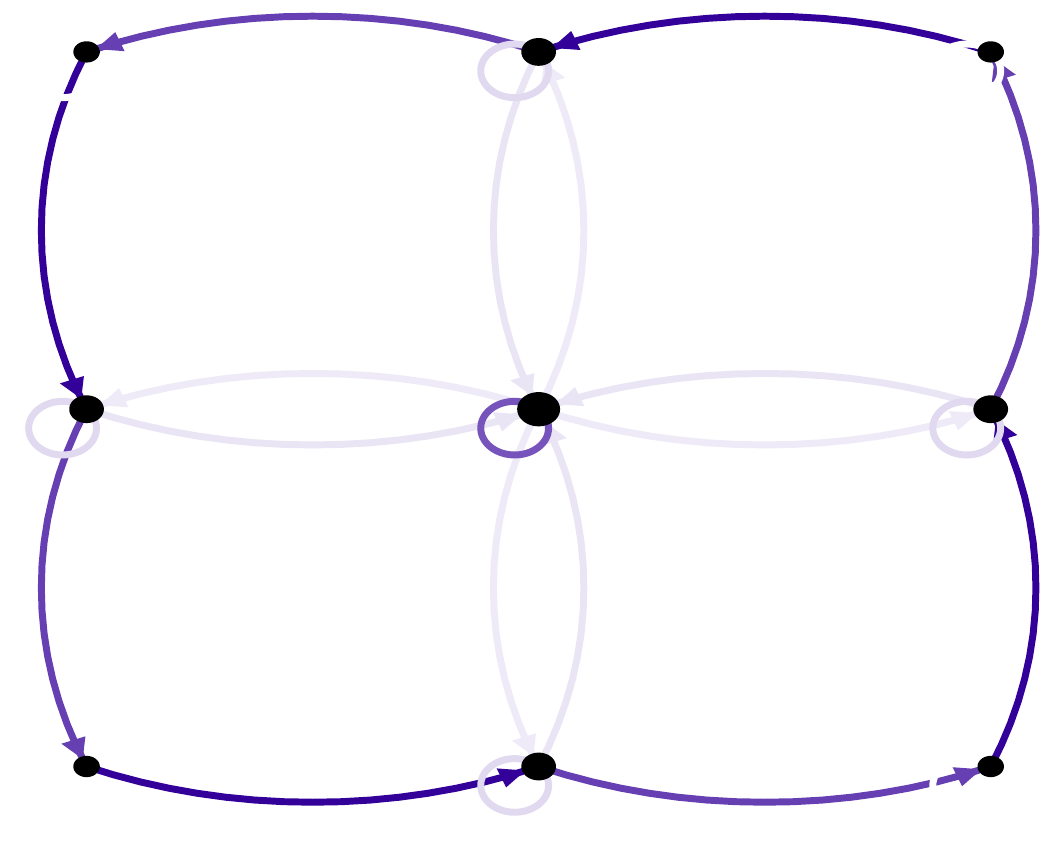}}
\caption{Random walk evader and optimal pursuer in grid}\label{fig:grid_RW}
\end{center}
\end{figure}

\begin{figure}[htbp]
\begin{center}
\subfigure[ Unpredictable evader chain]{
\includegraphics[scale=0.38]{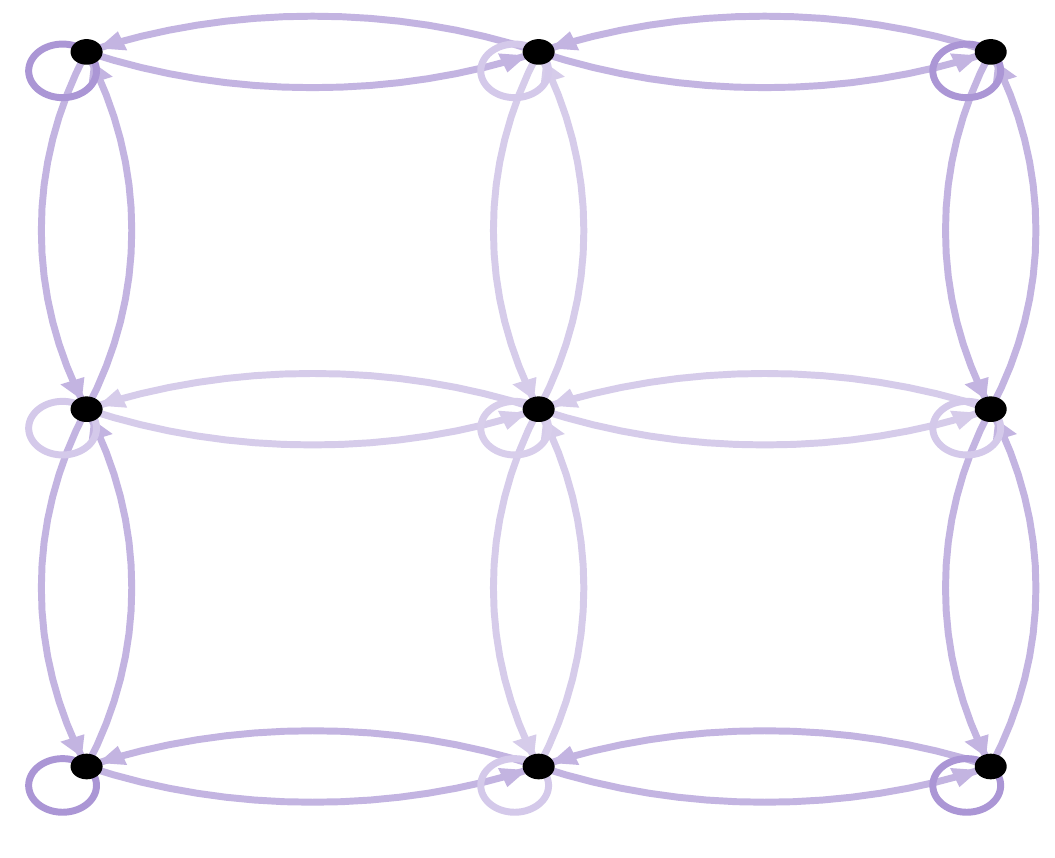}}
\subfigure[ Optimal pursuer chain]{
\includegraphics[scale=0.38]{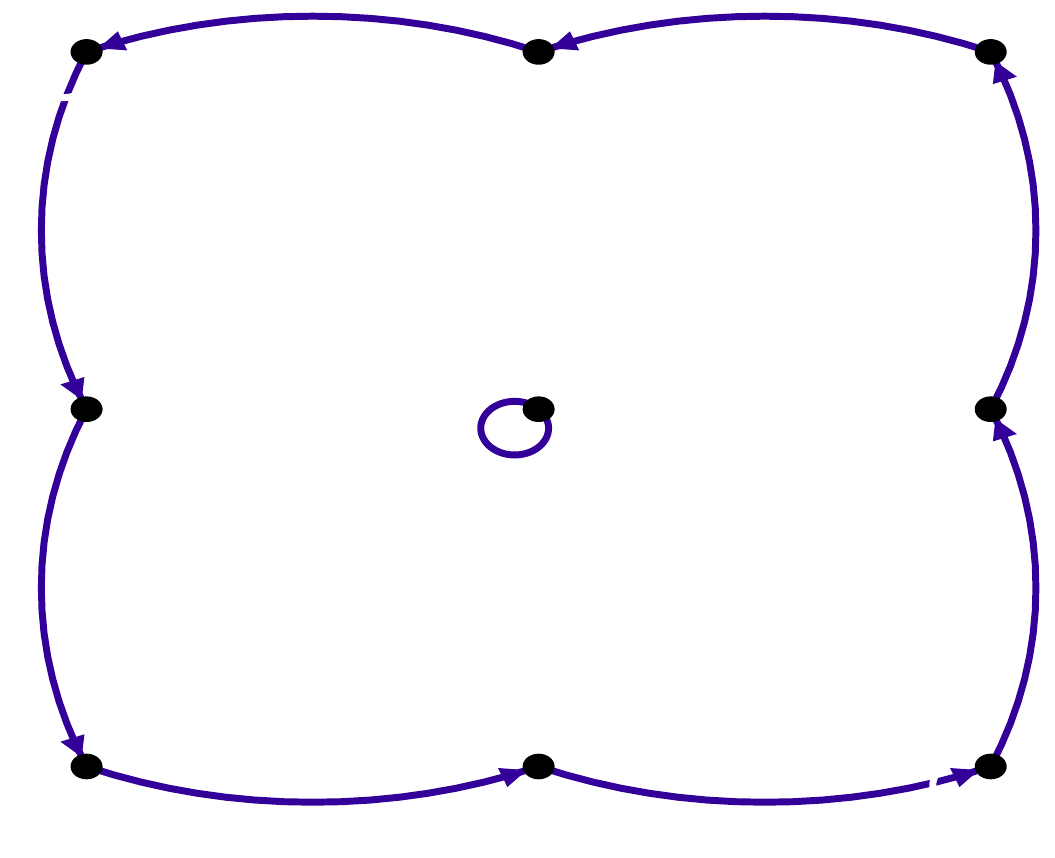}}
\caption{Unpredictable evader and optimal pursuer in grid}\label{fig:grid_entropy}
\end{center}
\end{figure}

In contrast, when the evader moves around fast enough, the optimal pursuer almost stays stationary and waits to be hit by the evader as shown in Fig.~\ref{fig:grid_kemeny}. The above observations in the grid graph are qualitatively consistent with those in the ring and complete graphs.

\begin{figure}[http]
\begin{center}
\subfigure[ Fast evader chain]{
\includegraphics[scale=0.38]{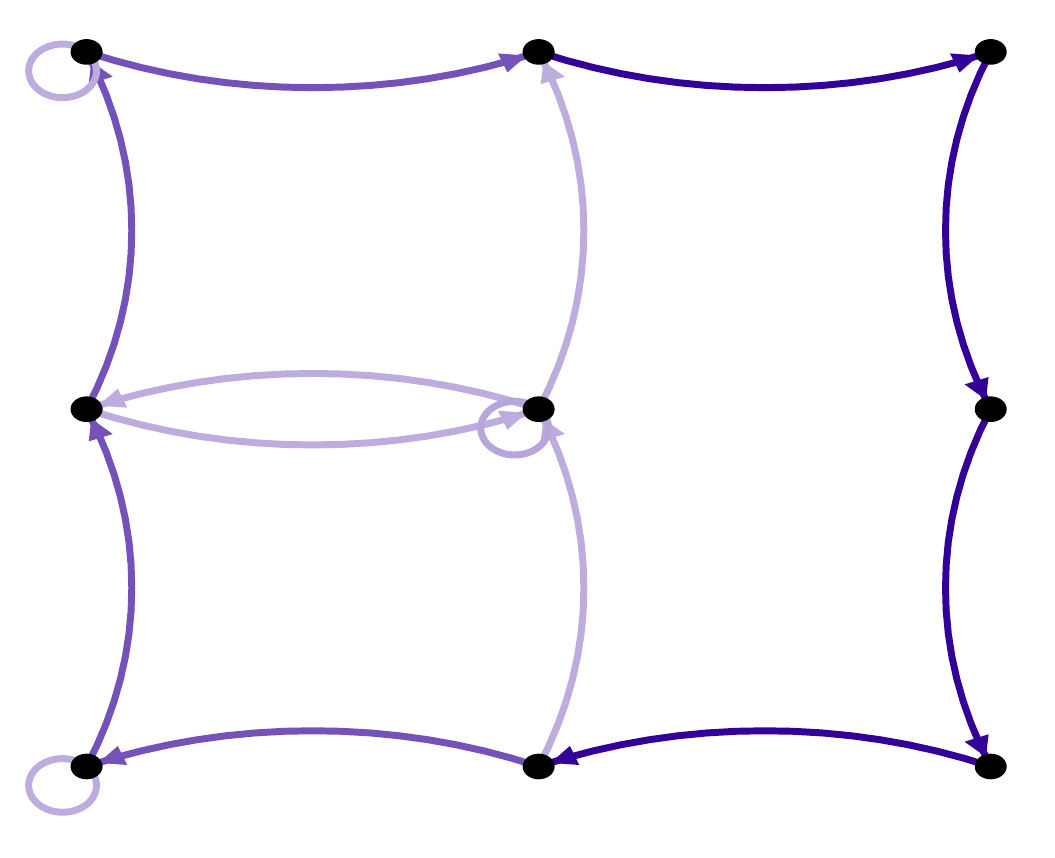}}
\subfigure[ Optimal pursuer chain]{
\includegraphics[scale=0.38]{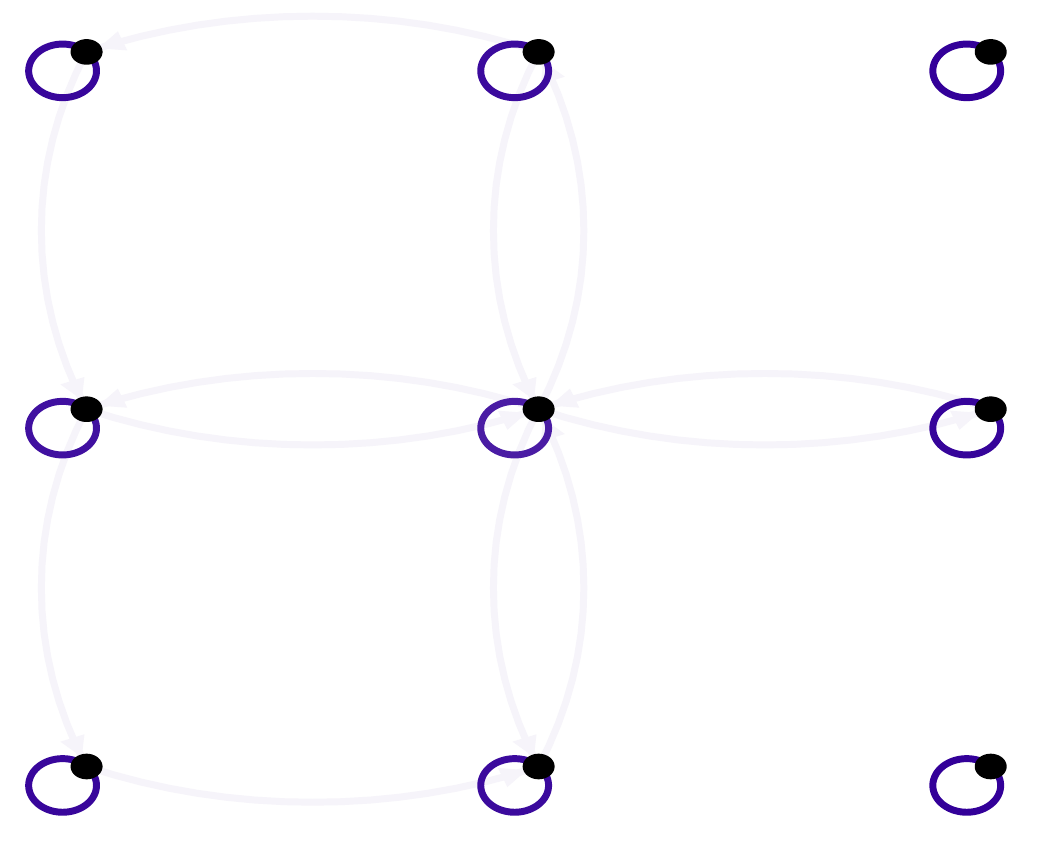}}
\caption{Fast evader and optimal pursuer in grid}\label{fig:grid_kemeny}
\end{center}
\end{figure}

\section{Conclusions}\label{sec:conclusion}
In this paper, we studied the expected meeting time of a single pursuer and a single evader moving on a graph according to discrete-time Markov chains. We presented novel closed-form expressions for the meeting times and  necessary and sufficient conditions for their finiteness. Then, we also discussed the connections with the hitting times of Markov chains. We finally formulated an optimization problem to obtain the optimal strategy for the pursuer faced with a mobile evader. Numerical examples were provided to explain the concepts and illustrate the results.

An interesting extension of the work discussed here would be to consider walkers moving with travel times similar to the cases studied in~\cite{XD-MG-FB:17o} and \cite{RP-PA-FB:14b}. 

\bibliographystyle{plainurl+isbn}
\bibliography{alias,Main,FB}
\end{document}